\documentclass[a4paper]{llncs}
                              \institute{ }
\titlerunning{Compiling TM to SMM}
\usepackage{amsmath}
\usepackage{breakcites}
\usepackage{paralist}
\usepackage{graphicx}
\usepackage{listings}
\usepackage[T2A]{fontenc}
\usepackage[utf8]{inputenc}
\usepackage[russian, english]{babel}

\newtheorem{thesis}{Thesis}
\author{J.-M. Chauvet}
\date{\today}
\title{Multiway Storage Modification Machines}
\begin{document}

\maketitle
\begin{abstract}
We present a parallel version of Schönhage's Storage Modification Machine, the Multiway Storage Modification Machine (\emph{MWSMM}). Like the alternative Association Storage Modification Machine of Tromp and van Emde Boas, MWSMMs recognize in polynomial time what Turing Machines recognize in polynomial space. Falling thus into the Second Machine Class, the MWSMM is a parallel machine model conforming to the parallel computation thesis. We illustrate MWSMMs by a simple implementation of Wolfram's String Substitution System.
\end{abstract}

\section{Introduction}
\label{sec:orga45f054}
The Storage Modification Machine (\emph{SMM}) is a machine model introduced by Schönhage in 1980 \cite{Schoenhage1980}. A somewhat similar computation model was previously discussed by Колмогоров and Успенский \cite{ref21} in 1957, the Kolmogorov-Uspensky Model (\emph{KUM}). A SMM represents a single computing agent. It is equipped with memory and a processing unit. Its memory stores a finite directed graph of equal out-degree nodes with a distinguished one called the \emph{center}. (Edges of this graph are also called \emph{pointers}.) Edges leading out of each node are uniquely labelled by distinct \emph{directions}, drawn from a finite set \(D\). In addition, like in the Random Access Machine (\emph{RAM}) and Random Access Stored Program (\emph{RASP}) models, a SMM has a stored program, or control list, consisting of instructions to operate on the directed graph by creating new nodes and (re)directing pointers.

Any string \(x \in D^*\) refers to the node \(p(x)\) reached from the center by following the sequence of directions labelled by \(x\). The restricted instruction set is as follows:

\begin{itemize}
\item \emph{\textbf{new} label} creates a new labelled node and makes it the center, setting all its outgoing edges to the previous center. (In another variant, the newly created node does not become the center.) Without a label specification, a unique ID is generated instead.
\item \emph{\textbf{set} xd \textbf{to} y} where \(x,y\) are paths in \(D^*\) and \(d \in D\) is a direction, redirects the \(d\) edge of \(p(x)\) to point to \(p(y)\).
\item \emph{\textbf{center} x} where \(x\) is a path, moves the center to \(p(x)\).
\item \emph{\textbf{if} x y \textbf{then} ln} where \(x,y\) are paths and \emph{ln} a line number, jumps to line \emph{ln} if \(p(x) = p(y)\) and skips to the next line if not. Line numbers can be absolute, \emph{ln}, or relative to the current line number, \emph{+rel} or \emph{-rel}.
\item \emph{\textbf{stop} message} halts the SMM, printing out \emph{message}.
\end{itemize}

It has been established that from the perspective of computational complexity theory the SMM (if equipped with the correct space measure, see \cite{VANEMDEBOAS1989103} and \cite{LUGINBUHL1993} for a review) is computationally equivalent to the other standard sequential machine models like the Turing Machine and the RAM models. In particular, SMM and Turing Machines can simulate each other \cite{chauvet21:_compil_turin_machin_storag_modif_machin}. These equivalences are indicative of the \emph{Invariance Thesis} for sequential machine models \cite{SLOT1988}:

\begin{thesis}[Invariance Thesis]
There exists a standard class of machine models, which includes among others all variants of Turing machines, all variants of RAMs and RASPs with logarithmic time and space measures, and also the RAMs and RASPs in the uniform time and logarithmic space measure, provided only standard arithmetical instructions of additive type are used. Machine models in this class simulate each other with polynomially bounded overhead in time and constant factor overhead in space.
\label{orgd09e194}
\end{thesis}

As remarked by Slot and van Emde Boas \cite{SLOT1988} "\emph{the thesis becomes a guiding rule for specifying the right class of models rather than an absolute truth,}" and is used to define so-called \emph{reasonable machines}. For the SMM time measure, the only proper candidate is the uniform time measure. (Even if each instruction is charged by the lengths of the paths it involves, these lengths being independent from the current graph structure at runtime, the weight would differ from the uniform by no more than a constant factor determined by the longest paths in the program.) For the SMM space measure, Schönhage's estimates of the number of graphs of \(n\) nodes over an alphabet of directions of size \(d\):

\begin{equation}
n^{nd-n+1} \leq N_{n,d} \leq \binom{nd}{n} \frac{nd-n+1}{nd+1}
\end{equation}

are suggestive that: on \(n\) nodes of a SMM, one can encode \(\mathcal{O}(nd\log{}n)\) bits, compared to the \(\mathcal{O}(n\log{}d)\) bits on \(n\) tape cells of a TM with the same size \(d\) alphabet. This result was strengthened by Lunginbuhl and Loui in \cite{LUGINBUHL1993} where a real-time simulation of a TM of space complexity \(s\) by a SMM using \(\mathcal{O}(s/{\log{}s})\) nodes is presented. Two notions of space measures are discussed: the \emph{mass} as the number of \texttt{new} instructions executed before halting, i.e. the number of nodes created during execution--which may be greater than the number of nodes reachable from the center-- and the \emph{capacity} as \(dn\log{}n\), where \(d\) is the size of the set of directions (number of pointers) and \(n\) the number of nodes. Based on these original observations, the SMM can be simulated with constant-factor space overhead on a TM and vice versa.

For most sequential models, \emph{parallel machine models} have been proposed based on the classical sequential version.

\begin{center}
\begin{tabular}{lll}
Sequential model & Parallel Model & Reference\\
\hline
TM & Recursive branching & Savitch W. J. \cite{savitch1977recursive}\\
TM & Forking & Wiedermann J. \cite{RUUCS8411}\\
RAM & Numerous variations & Survey: \cite{VANEMDEBOAS19901}\\
SMM & Associative SMM & Tromp J. and van Emde Boas P. \cite{DBLP:conf/dagstuhl/TrompB92}\\
\hline
\end{tabular}
\end{center}

Most parallel models in the table above, with the exception of Wiedermann's PTM for instance, belong to a \emph{Second Machine Class}, a term introduced by van Emde Boas to designate machine models obeying the \emph{Parallel Computation Thesis}.

\begin{thesis}[Parallel Computation Thesis]
Whatever can be solved in polynomially bounded space on a reasonable sequential machine model can be solved in polynomially bounded time on a reasonable parallel machine and vice versa.
\label{orge0fbb2c}
\end{thesis}

As \emph{reasonable} sequential machines, namely these that obey the Invariance Thesis, constitute the First Machine Class, \emph{reasonable} parallel machines, those complying with the Parallel Computation Thesis, constitute the Second Machine Class \cite{VANEMDEBOAS19901}.

As mentioned above, not all parallel models follow the Parallel Computation Thesis, so that this thesis only represents a frequently occurring version of the uniform parallelism power but, according to van Emde Boas, one that strikes the right balance between "\emph{exponential growth capability and proper degree of uniformity.}"

In this paper we propose an alternate parallel version of the SMM to Tromp and van Emde Boas Associative Storage Modification Machine (\emph{ASMM}), called Multiway Storage Modification Machine (\emph{MWSMM}). The computational power of the MWSMM originates from a new instruction, \texttt{match}, triggering a multi-centered updating process of the graph structure. Like the ASMM, the MWSMM model may be considered to be a member of the class of sequential machines which operate on large objects in unit time and obtain their power of parallelism thereof. A simple deterministic updating process for the MWSMM is shown to locate this model into the Second Machine Class.

\section{Multiway Storage Modification Machines}
\label{sec:org13a2230}
In this section we propose a parallel version of the SMM model as an alternative to the ASMM of \cite{DBLP:conf/dagstuhl/TrompB92}. Its source of inspiration comes from  \emph{Production Systems} an Artificial Intelligence architecture \cite{Waterman:78} which consists primarily of a set of rules about behavior, but also includes the mechanism necessary to follow those rules as the system responds to changes of state. Those rules, termed productions, are a basic knowledge representation scheme found useful in automated planning, expert systems and action selection tasks.

Along these lines we augment the sequential SMM instruction set with one control:
\begin{itemize}
\item \emph{\textbf{match} bool-expr} where \emph{bool-expr} is a boolean expression on paths. This instruction works like a production. It selects any node in the graph which, were it the center, would make the boolean expression true. The instructions that follow the \texttt{match} are then executed in parallel on all nodes thus selected, as if each one of the selected node was the center.
\end{itemize}

The \texttt{match} control semantics is similar to the \emph{recognize-act cycle} typically found in production and rule-based systems. Candidate center nodes are selected that set to true the boolean expression, and the rest of the program is executed on all nodes as if they were the center. The principle is similar to \emph{forking} the sequential execution of the program but starting from several nodes rather than from a single center.

\begin{figure}[htbp]
\centering
\includegraphics[width=.9\linewidth]{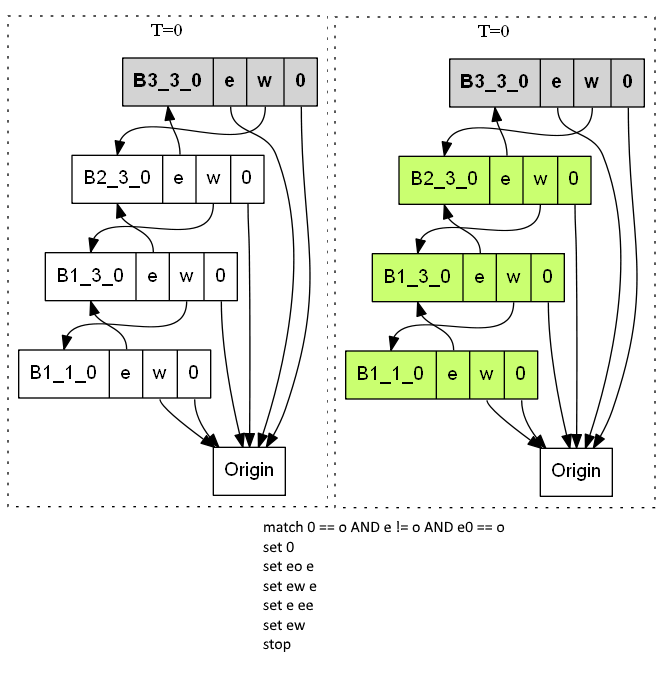}
\caption[\texttt{match}]{\label{fig:org487647b}Top left, an example graph of a SMM before a \texttt{match} instruction. Top right the selected nodes, in green, after selection by the \texttt{match} instruction stated in the code fragment at bottom. The SMM center is colored gray and in bold.}
\end{figure}

In Figure \ref{fig:org487647b} we show the selection operated by a \texttt{match} instruction. In this example machine, three nodes end up selected and the  six instructions following \texttt{match} are executed in parallel with the three selected nodes considered as center.

\subsection{The Structure of a Multiway SMM program}
\label{sec:orgaf46231}
Parallelism is introduced in the MWSMM at each \texttt{match} instruction of its program. The execution forks into parallel execution of the rest of the instructions following the \texttt{match}, or stops if no node is found to match the boolean expression. The rest of the instructions may itself contain further \texttt{match} instructions, triggering another forking of parallel executions, or \texttt{stop} instructions terminating the forked execution (as in Figure \ref{fig:org487647b}). Note that the execution is deterministic as all other instructions always denote one node.

More generally a MWSMM program is segmented in a \emph{prologue} control list followed by one or several \emph{match-blocks} containing control lists ending on a \texttt{stop} instructions. The prologue control list builds up the initial graph of the SMM; it is usually executed once at start-time, and may be considered similar to a global data segment in an assembly program. A distinguished list of match-blocks are specified as \emph{top-level} in the invocation of the SMM. The top-level match-blocks are executed in parallel during a single \emph{run} of the SMM. If none of the top-level \texttt{match} statements select one or several nodes, the run terminates.

\begin{figure}[htbp]
\centering
\includegraphics[width=.9\linewidth]{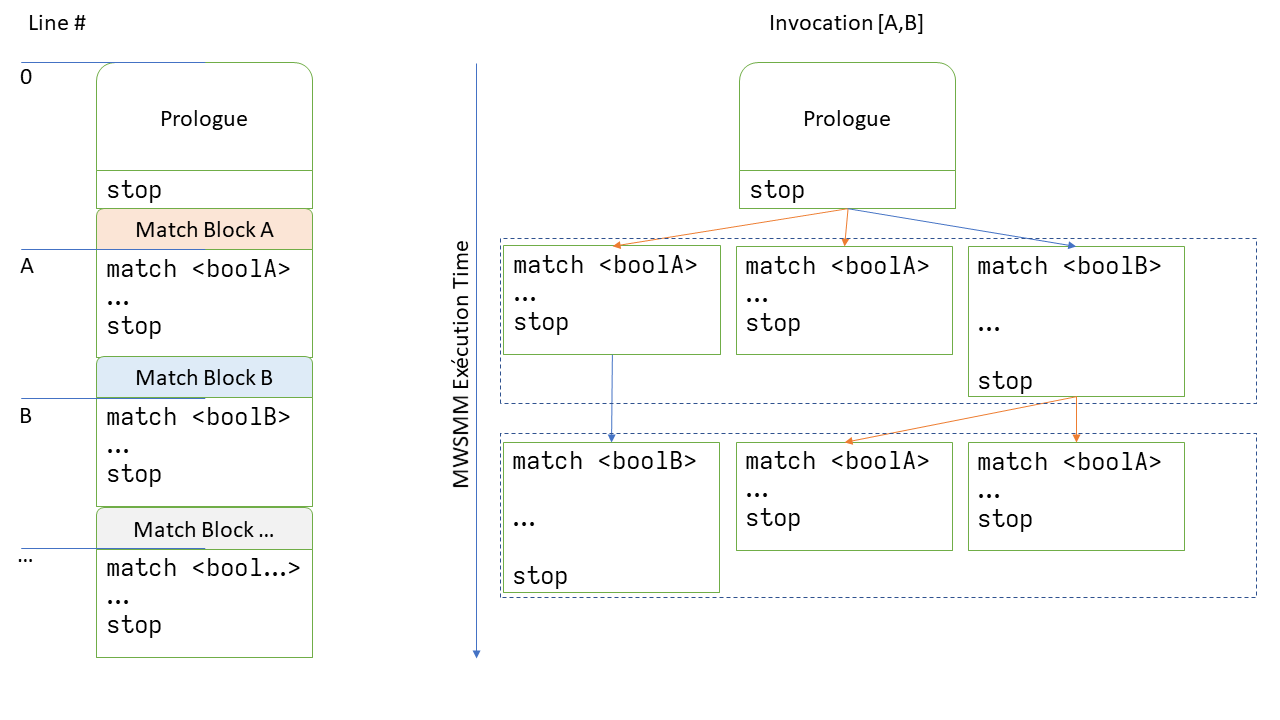}
\caption[\emph{B}]{\label{fig:orgd3f53db}The structure of a MWSMM program. Left the program as a sequentially numbered list of statements delimiting a prologue block and one or several match blocks \emph{A}, \emph{B}\ldots{} Right: at runtime, on invocation with top-level blocks [A, B] the prologue is executed first and match-blocks \emph{A} and \emph{B} are executed in parallel, e.g. twice for \emph{A} when two nodes match the \emph{A} boolean expression immediately after the prologue and once for \emph{B} as one node matches the \emph{B} boolean expression immediately after the prologue. From the graph resulting from the parallel executions, new nodes may enter the selection for \emph{A} and \emph{B} boolean expressions, which trigger another parallel execution of the corresponding blocks, until selection is empty for all top-level blocks.}
\end{figure}
\newpage

\subsection{MWSMMs obey the Parallel Computation Thesis}
\label{sec:org9389f43}
We sketch the proof of the membership of MWSMMs into the Second Machine Class the usual way, using the \emph{Quantified Boolean Formula} (QBF) satisfaction problem as a representative PSPACE-problem. The proof is in two parts \cite{VANEMDEBOAS19901}:

\begin{lemma}[$NPTIME(MWSMM) \subseteq PSPACE$]
The language recognized by a non-deterministic MWSMM in polynomial time is recognized by the standard model of off-line multi-tape TM in polynomial space.
\label{orgfeac9ff}
\end{lemma}

\begin{proof}[Sketch]
It is easy to see that \(t\) steps of a non-deterministic MWSMM can be generally simulated on a TM with \(\mathcal{O}(t^2)\) space. Following \cite{DBLP:conf/dagstuhl/TrompB92}, we define three arrays \emph{instr}, \emph{nodes} and \emph{center} which respectively hold the instruction, the number of nodes and the center of the MWSMM at successive steps (times) \(i = 0...t\). Each element of the array fits into \(t\) bits since the number of nodes at most doubles (when all nodes are a match) at each step; every node has a unique number.

As a result of a successful \texttt{match} control, following instructions in the match block are executed in sequence \(n\) times, where \(n \leq t\) is the number of matching nodes. Each sequence is started with \emph{center[i+1] = nodes[k]} where \(k\) is the unique number of the node in the matching set.

Individual instructions of the MWSMM are each simulated by simple transitions in the TM reading and writing data from these three arrays. The number of the TM steps for a given MWSMM instruction is bounded by the maximum path length in the MWSMM program. With the above arrays, the simulating TM requires only \(\mathcal{O}(t^2)\) space.
\end{proof}

\begin{lemma}[$PSPACE \subseteq PTIME(MWSMM)$]
The language recognized by the standard model of off-line multi-tape TM in polynomial space is recognized by a deterministic MWSMM in polynomial time.
\label{orgd3b582b}
\end{lemma}

\begin{proof}[Sketch]
We give a MWSMM program solving the QBF PSPACE-problem in MWSMM polynomial time, exploiting the implicit parallel power of the \texttt{match} instruction.

We consider a QBF with its two constitutive parts: a portion containing only quantifiers \(Q_1x_1...Q_nx_n\) where the \(Q_i\) are quantifiers \(\exists\) or \(\forall\), and the \(x_i\) are the logical variables; another containing an unquantified boolean expression involving only the binary \emph{and} and \emph{or} operators or the unary negation \emph{not} on the \(x_i\) logical variables.

In the proposed construction, the prologue builds a tree representation of the quantifier portion of the QBF. The program then lists two match-blocks: one to root a tree representation of the formula part of the QBF (with \emph{and}, \emph{or} and \emph{not} nodes), in parallel at each one of the leaves of the tree built in the prologue; another to compute, in parallel, the result of the execution of the boolean formula on all \(2^n\) logical value combinations of the \(x_i\) variables and propagate bottom-up the final value of the QBF satisfaction. The existence of a solution to the QBF problem is read from one distinguished direction of the center when the MWSMM stops.

In the MWSMM graph, each node is \emph{typed} as it represents one of:
\begin{itemize}
\item a universal quantifier \(\forall\) in the prologue quantifier tree;
\item an existential quantifier \(\exists\) in the prologue quantifier tree;
\item a binary \emph{and} operator in the formula subtree;
\item a binary \emph{or} operator in the formula subtree;
\item a unary \emph{not} operator in the formula subtree;
\item a variable \(x_i\) in the formula subtree.
\end{itemize}

The type is classically encoded in several dedicated directions (\texttt{A}, \texttt{S}, and \texttt{N} respectively for \emph{all}, \emph{some} and \emph{not}), pointing either to self or to a constant origin node pointed to by a special direction (\texttt{o}) in every node. MWSMM nodes keep track of their position in the tree and subtrees with three directions (\texttt{l}, \texttt{r} and \texttt{p} for left child, right child and parent). Finally each node has \(n\) directions, one for each variable, in which the combination of logical values of variables \(x_i\) to be tested is encoded the usual way (pointing to self for \texttt{False} or to the origin for \texttt{True}), and a value direction (\texttt{v}) encoding the node's own logical result value in the same way. 

Two ancillary "binary" directions (\texttt{x} and \texttt{i}) are used in the last evaluation match-block to indicate respectively that the node value is ready to be computed from the children node input values, and that the node is a leaf in the quantifier tree. The MWSMM then has \(10+n\) directions with \(n\) the number of variables in the QBF. (This number could be reduced by reusing directions to different purposes according to the type of the node.)

The length of the prologue is proportional to the number of variables \(n\) in the QBF, the graph it recursively builds has \(2^{n+1} - 1\) MWSMM nodes, \(2^n\) of which represent the leaves of the quantifier-part tree representation.

The length of the first match-block is proportional to the nesting depth \(d\) of the formula-part of the QBF and it recursively builds a subtree of \(2^{d+1} - 1\) MWSMM nodes attached in parallel to the above \(2^n\) leaves in the graph at the end of the prologue.

The length of the second match-block is proportional to the sum of the depth \(d\) and \(n\), the number of variables, as it propagates bottom-up the binary values, resulting from applying the logical operators, up all formula subtrees and the quantifier tree in parallel.

Overall the resolution MWSMM time is proportional to the lengths of prologue and match blocks, i.e. on the total depth of the QBF, \(d+n\). This polynomial bound is due to the parallel execution of the QBF formula on all combinations of logical values for variables. The exhaustive combinations are  built in the first match-block and evaluated in the second one.
\end{proof}

\begin{figure}[htbp]
\centering
\includegraphics[width=.9\linewidth]{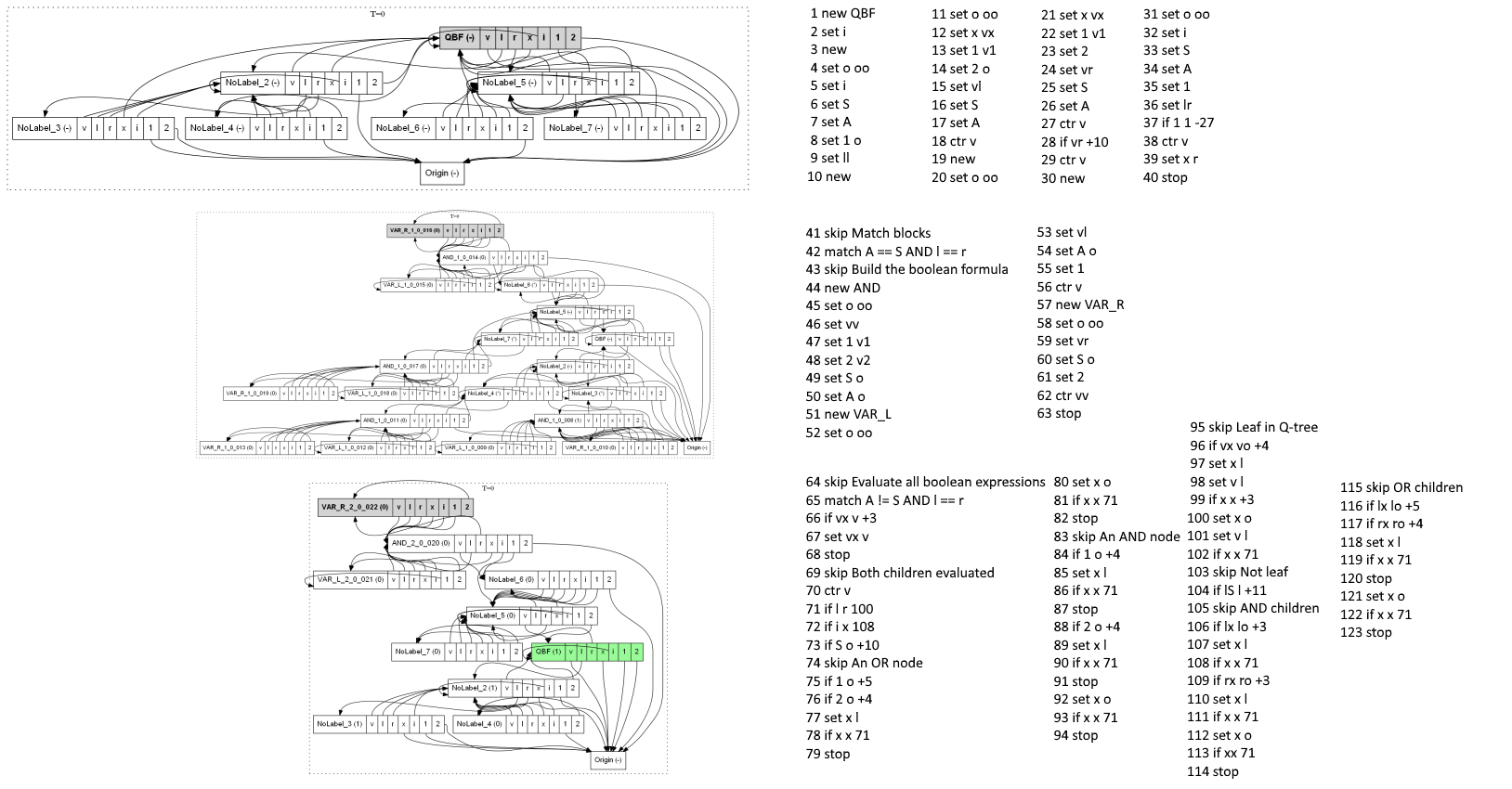}
\caption{\label{fig:orgcdefde6}Example MWSMM-run for the QBF-satisfaction problem on the simple formula \(\exists x_1 \exists x_2 (AND x_1 x_2)\). Top: MWSMM graph, left, after the prologue line 1 to 40, right, is run. Middle: MWSMM graph, left, after the first match-block line 41 to 63, is run in parallel. Bottom: MWSMM graph, left, after the second match-block line 64 to 123 is run. The final QBF state, in green, is read in direction \texttt{v} of the root \emph{QBF} labelled node, here \emph{1} i.e. \texttt{True}. (Center is highlighted in gray and bold characters.) Note that \texttt{match} conditions are exclusive and that the last match-block (bottom) selected nodes are built by the first match-block (middle) so that it is executed after the first one. Each match-block is executed in parallel on all nodes it selects as local centers. Here then the prologue is executed once, and each match-block is executed in parallel, one after the other; all blocks have MWSMM instruction lengths proportional to at most \(n+d\), with \(n\) the number of QBF variables and \(d\) the depth of the QBF formula. Maximum space extension is reached at the end of the execution of the first match-block (middle). When the machine halts, the execution of the last match-block (bottom) has made some nodes unreachable from the center.}
\end{figure}
\newpage

\section{String Substitution Systems}
\label{sec:org1e33a79}
In the previous section we considered a MWSMM with a deterministic updating process comparable to the inference engine found in the early \emph{expert system} AI architecture. Each match-block in the MWSMM program is handled as a rule in a production system. With this updating process the MWSMM falls into the Second Machine Class.

\subsection{Non-Deterministic Updating Process}
\label{sec:orgc97e656}
Following the lines of the Hardware Modification Machine (\emph{HMM}) introduced in \cite{Dymond1986a} we now look into a non-deterministic updating process by considering each match-block sequence of instruction as the program of an individual MWSMM. The current graph is its initial configuration and its center is preset to one of the matching nodes. Like the HMM, the MWSMM can be thought of as a collection of multiple-input finite-state transducers, except that (i) interconnections between units are rearranged by the implicit sequence of match-block executions and (ii) new units are activated as the computation proceeds after selection by the \texttt{match} instruction.

In a non-deterministic MWSMM there are often multiple ways in which a given \texttt{match}  can succeed and multiple ways in which blocks are executed, as in the case of several successful matches from different match-blocks in the same program. Taking our hints from \cite{wolfram2020project} we consider in this section the simple case of String Substitution Systems. These involve strings of a given alphabet whose elements are repeatedly replaced according to a finite set of rules.

A String Substitution System is simulated by a MWSMM where the string is represented by a doubly-linked list of nodes (using conventionally the \texttt{e} and \texttt{w} directions, \emph{east} and \emph{west}) and an encoding of the character on \(\log{}n\) special directions for an alphabet of size \emph{n}, each of which pointing either to self or to a fixed origin node. Each rule then appears as a match-block in the MWSMM program, which first selects where to apply the transformation and then executes it.

Let us consider, for instance, the first example in Chapter 5 of \cite{wolfram2020project}, the substitution system with rules \((A \rightarrow BBB, BB \rightarrow A)\), starting with the string \emph{A}.
\newpage

\begin{table}
\scriptsize
\begin{center}
\begin{tabular}{lll}
Prologue & \(A \rightarrow BBB\) & \(BB \rightarrow A\)\\
\hline
1 \texttt{skip A->BBB, BB->A; A} & 5 \texttt{match 0 == \_} & 23 \texttt{match 0 == o AND e != o AND e0 == o}\\
2  \texttt{new A} & 6  \texttt{ren B} & 24  \texttt{ren A}\\
3  \texttt{set 0} & 7  \texttt{set 0 o} & 25  \texttt{set 0}\\
4  \texttt{stop} & 8  \texttt{new B} & 26  \texttt{set ew e}\\
 & 9  \texttt{set o oo} & 27  \texttt{set e ee}\\
 & 10  \texttt{set 0 o} & 28  \texttt{set ew}\\
 & 11  \texttt{set e we} & 29  \texttt{stop}\\
 & 12  \texttt{set we} & \\
 & 13  \texttt{if e o +2} & \\
 & 14  \texttt{set ew} & \\
 & 15  \texttt{new B} & \\
 & 16  \texttt{set o oo} & \\
 & 17  \texttt{set 0 o} & \\
 & 18  \texttt{set e we} & \\
 & 19  \texttt{set we} & \\
 & 20  \texttt{if e o +2} & \\
 & 21  \texttt{set ew} & \\
 & 22  \texttt{stop} & \\
\end{tabular}
\end{center}
\caption[\texttt{ren}]{\label{org085e779}The alphabet is  \(\{A, B\}\), we need a single direction, \texttt{0}, to discriminate an \emph{A} or a \emph{B}. The prologue block, line 1-4, creates the initial string \emph{A} (setting \texttt{0} to self). The first rule, line 5-22, applies to any \emph{A}, here identified by any node in the graph, with direction \texttt{0} pointing to self (while \emph{B} would be indicated by pointing to the origin). The second rule, line 3-29, recognizes the substring \emph{BB} by checking that the current direction \texttt{0}  points to the origin (always pointed to by the reserved \texttt{o} direction) and that the eastern-neighbor is also a \emph{B}. The simulation-specific \texttt{ren} instruction simply changes the label of the node for output readability.}
\end{table}

Table \ref{org085e779} shows how each rule of the string substitution system is coded into a separate match-block while the initial string is built by the prologue of this MWSMM program. Using the non-deterministic updating process and running the MWSMM for three steps yield two possible states, both presented in Figure \ref{fig:orgb037669}. These are the leaves of the evolution tree, which retraces all possible substitutions at each step, as presented in \ref{fig:org8ebf724}.

\begin{figure}[htbp]
\centering
\includegraphics[width=.9\linewidth]{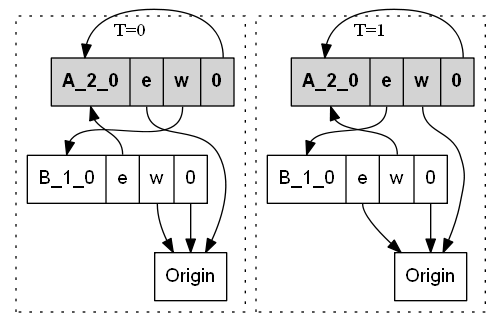}
\caption[\emph{BA}]{\label{fig:orgb037669}The two MWSMM graphs resulting from executing 3 steps of the string substitution system \((A \rightarrow BBB, BB \rightarrow A)\) on the initial string \emph{A}. They read: \emph{AB} and \emph{BA} respectively.}
\end{figure}

\begin{figure}[htbp]
\centering
\includegraphics[width=5cm]{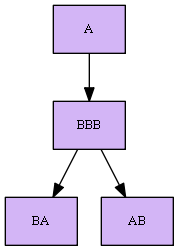}
\caption{\label{fig:org8ebf724}The natural representation of a the evolution of the string substitution system as a multiway system in which there can be multiple outcomes at each step and there are multiple paths of evolution (here shown after 3 steps).}
\end{figure}
\newpage

\subsection{Recovering the Phenomenon of Causal Invariance}
\label{sec:org2d4d310}
The MWSMM simulation recovers the causal invariance phenomenon: consider the string substitution system with the single rules \((BA \rightarrow AB)\) starting this time from \emph{BABABA}. The program for the simulation is presented in Table \ref{org0002c67}.

\begin{table}
\scriptsize
\begin{center}
\begin{tabular}{llll}
Prologue (1/3) & Prologue (2/3) & Prologue (3/3) & \(BA \rightarrow AB\)\\
\hline
1 \texttt{skip} & 10 \texttt{set o oo} & 19 \texttt{new B} & 30 \texttt{match 0 == o AND e != o AND e0 == e}\\
2 \texttt{new B} & 11 \texttt{set e we} & 20 \texttt{set o oo} & 31 \texttt{ren A}\\
3 \texttt{set 0 o} & 12 \texttt{set we} & 21 \texttt{set e we} & 32 \texttt{set 0}\\
4 \texttt{new A} & 13 \texttt{set 0 o} & 22 \texttt{set we} & 33 \texttt{ctr e}\\
5 \texttt{set o oo} & 14 \texttt{new A} & 23 \texttt{set 0 o} & 34 \texttt{ren B}\\
6 \texttt{set e we} & 15 \texttt{set o oo} & 24 \texttt{new A} & 35 \texttt{set 0 o}\\
7 \texttt{set we} & 16 \texttt{set e we} & 25 \texttt{set o oo} & 36 \texttt{stop}\\
8 \texttt{set 0} & 17 \texttt{set we} & 26 \texttt{set e we} & \\
9 \texttt{new B} & 18 \texttt{set 0} & 27 \texttt{set we} & \\
 & 28 \texttt{set 0} &  & \\
 & 29 \texttt{stop} &  & \\
\end{tabular}
\end{center}
\caption[\emph{B}]{\label{org0002c67}The alphabet is again \(\{A, B\}\). The prologue is a bit longer as it creates a string of alternating \emph{A} and \emph{B} symbols. The match-block for the single rule of this substitution system is however simpler here: we simply switch the encodings of the symbol at the center and its eastern neighbor.}
\end{table}

Running the simulation show that there are different possible paths through the evolution graph leading to different histories for the system, see Figure \ref{fig:orgf876720} left. But now all these paths converge to a final state. (The interpretation being that the rule is effectively sorting As in front of Bs by repeated permutations.) In the MWSMM implementation, this translates to all choices of selected nodes at each firing of the \texttt{match} instruction leading to the same MWSMM graph and center.

\begin{figure}[htbp]
\centering
\includegraphics[width=.9\linewidth]{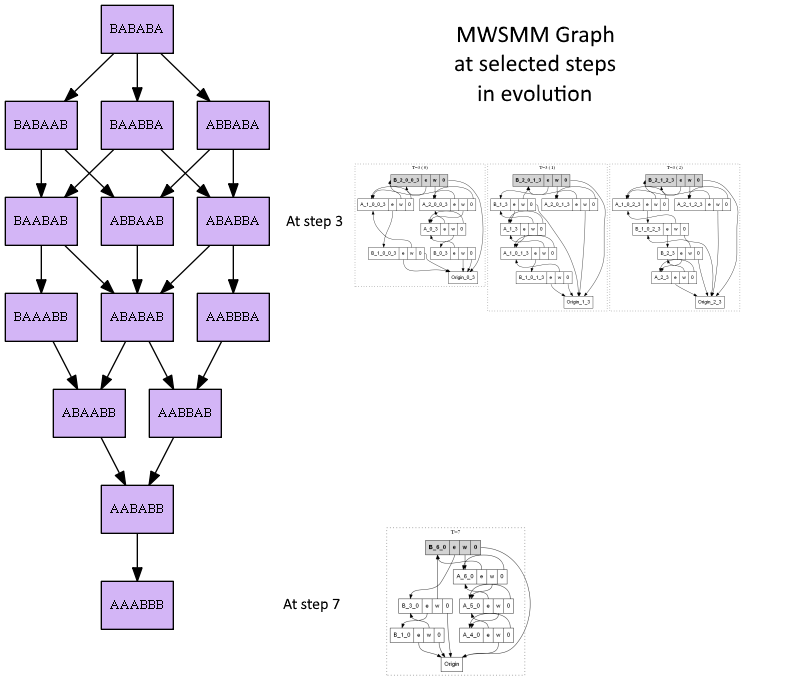}
\caption[\emph{7}]{\label{fig:orgf876720}An example string substitution system allowing match-block executions to be carried out in different sequences, generating different intermediate MWSMM graphs while always ending in the same final state. Such systems are called \emph{causal invariant} in \cite{wolfram2020project}, and totally causal invariant when the convergence happens whatever the initial string is. Left: the evolution tree starting from \emph{BABABA}. Right: Details of the possible intermediate MWSMMs at step \emph{3}, and, bottom, the single ending MWSMM at step \emph{7}.}
\end{figure}
\newpage

\section{Conclusions}
\label{sec:orga0559db}
The MWSMM is another parallel model based on Schönhage's Storage Modification Machine. In contrast to the ASMM which derives its power from the use of associative addressing \cite{DBLP:conf/dagstuhl/TrompB92} -- in which the first argument path of the \texttt{new} instruction may use \emph{reversed} directions so that it may create several nodes -- the MWSMM may execute instruction blocks in parallel starting from several distinct centers selected by the \texttt{match} instruction. As shown by Lam and Ruzzo \cite{10.1145/72935.72946}, when nodes are made individually equivalent to a finite automata the model is similar to a restricted version of the PRAM model of Fortune and Willye \cite{10.1145/800133.804339}. Other possibilities mentioned in \cite{DBLP:conf/dagstuhl/TrompB92} include enhancing the ASMM with a \texttt{set} instruction also having associative indexing capabilities, but the authors see the potential conflict resulting from the same node being addressed by two distinct paths in the \texttt{set} instruction as a major drawback. The MWSMM also confronts this possible conflict when updating its graph according to the deterministic process. Strategies for \emph{conflict-resolution}, also a technical feature characterizing the early AI inference engines, delineate different variants of the deterministic MWSMM model. Yet another possibility is to augment both arguments of the \texttt{set} instruction with associative indexing capability. This allows to parallelize the update of propositional variables in an arbitrarily large set. This is exactly what the \texttt{match bool-expr} offers as the selective boolean expression is not restricted in any particular way\footnote{(Listings for the MWSMM programs are on-line at \url{https://github.com/CRTandKDU/SMM}.)}.

Classification of machine models is an important line of research in Complexity Theory. Cook \cite{Cook1981,Dymond1989},  has classified the synchronous parallel models according to whether the interconnection among processors during a computation is fixed or variable. Van Emde Boas \cite{VANEMDEBOAS19901} distinguishes models that obey the \emph{Invariance Thesis}, the \emph{Parallel Computation Thesis} or none. While the MWSMM model would fall into Cook's "variable interconnection" class, only when the updating process is deterministic does it fall into van Emde Boas's Second Machine Class. A non-deterministic MWSMM, on the other hand, simulates a string substitution system traversing its evolution graph to recover phenomena like causal invariance \cite{wolfram2020project}. These results, however, illustrate Hong's notions of similarity and duality \cite{Hong1984} where all well-known models of computations are all similar in the sense that their parallel time, their sequential time, and their space complexities are simultaneously polynomially related.

\bibliographystyle{plain}
\bibliography{smm}

\begin{thebibliography}{10}

\bibitem{chauvet21:_compil_turin_machin_storag_modif_machin}
J.~M. Chauvet.
\newblock {Compiling Turing Machines into Storage Modification Machines}, 2021.

\bibitem{Cook1981}
S.~A. Cook.
\newblock Towards a complexity theory of synchronous parallel computation.
\newblock {\em {L'Enseignement Mathématique}}, XXVII, 1981.

\bibitem{Dymond1986a}
P.~W. Dymond.
\newblock On nondeterminism in parallel computation.
\newblock {\em Theoretical Computer Science}, 47, 1986.

\bibitem{Dymond1989}
P.~W. Dymond and S.~A. Cook.
\newblock Complexity theory of parallel time and hardware.
\newblock {\em Information and Computation}, 80, 1989.

\bibitem{10.1145/800133.804339}
Steven Fortune and James Wyllie.
\newblock {Parallelism in Random Access Machines}.
\newblock In {\em Proceedings of the Tenth Annual ACM Symposium on Theory of
  Computing}, STOC '78, page 114–118, New York, NY, USA, 1978. Association
  for Computing Machinery.

\bibitem{Hong1984}
J.~W. Hong.
\newblock Similarity and duality in computation.
\newblock {\em Information and Control}, 62, 1984.

\bibitem{RUUCS8411}
J.Wiedermann.
\newblock {Parallel Turing Machines}.
\newblock Technical Report RUU-CS-84-11, Department of Information and
  Computing Sciences, Utrecht University, 1984.

\bibitem{ref21}
A.~N. Kolmogorov and V.~A. Uspenskii.
\newblock {On the definition of an algorithm}, 1957.

\bibitem{10.1145/72935.72946}
T.~W. Lam and W.~L. Ruzzo.
\newblock {The Power of Parallel Pointer Manipulation}.
\newblock In {\em Proceedings of the First Annual ACM Symposium on Parallel
  Algorithms and Architectures}, SPAA '89, page 92–102, New York, NY, USA,
  1989. Association for Computing Machinery.

\bibitem{LUGINBUHL1993}
D.R. Luginbuhl and M.C. Loui.
\newblock {Hierarchies and Space Measures for Pointer Machines}.
\newblock {\em Information and Computation}, 104(2):253--270, 1993.

\bibitem{savitch1977recursive}
Walter~J Savitch.
\newblock {Recursive Turing Machines}.
\newblock {\em International Journal of Computer Mathematics}, 6(1):3--31,
  1977.

\bibitem{Schoenhage1980}
A.~Schönhage.
\newblock {Storage Modification Machines}.
\newblock {\em {SIAM Journal on Computing}}, 9(3):490--508, 1980.

\bibitem{SLOT1988}
Cees Slot and Peter {van Emde Boas}.
\newblock The problem of space invariance for sequential machines.
\newblock {\em Information and Computation}, 77(2):93--122, 1988.

\bibitem{DBLP:conf/dagstuhl/TrompB92}
John Tromp and Peter van Emde~Boas.
\newblock Associative storage modification machines.
\newblock In Klaus Ambos{-}Spies, Steven Homer, and Uwe Sch{\"{o}}ning,
  editors, {\em Complexity Theory: Current Research, Dagstuhl Workshop,
  February 2-8, 1992}, pages 291--313. Cambridge University Press, 1992.

\bibitem{VANEMDEBOAS1989103}
Peter {van Emde Boas}.
\newblock {Space measures for Storage Modification Machines}.
\newblock {\em Information Processing Letters}, 30(2):103--110, 1989.

\bibitem{VANEMDEBOAS19901}
Peter {van Emde Boas}.
\newblock Chapter 1 - {Machine Models and Simulations}.
\newblock In JAN {Van Leeuwen}, editor, {\em Algorithms and Complexity},
  Handbook of Theoretical Computer Science, pages 1--66. Elsevier, Amsterdam,
  1990.

\bibitem{Waterman:78}
D.A. Waterman and F.~Hayes-Roth.
\newblock {\em Pattern Directed Inference Systems}.
\newblock Academic Press, New York, 1978.

\bibitem{wolfram2020project}
S.~Wolfram.
\newblock {\em {A Project to Find the Fundamental Theory of Physics}}.
\newblock Wolfram Media, Incorporated, 2020.

\end{thebibliography}
\end{document}